\documentclass[sigconf]{acmart}

\usepackage{booktabs} 
\usepackage{comment}
\usepackage{soul}
\usepackage{caption}
\usepackage{hyperref}
\usepackage{float}
\usepackage{amsfonts}
\usepackage[normalem]{ulem}
\usepackage{todonotes}
\usepackage{multirow}
\usepackage{subcaption,siunitx}
\usepackage{amsfonts}
\usepackage{amsmath,bm}
\usepackage{amsthm}
\usepackage{balance}

\usepackage{algorithm}
\usepackage{algpseudocode}
\usepackage{color}
\usepackage{array,booktabs,arydshln}

\copyrightyear{2019}
\acmYear{2019} 
\setcopyright{iw3c2w3}
\acmConference[WWW '19]{Proceedings of the 2019 World Wide Web Conference}{May 13--17, 2019}{San Francisco, CA, USA}
\acmBooktitle{Proceedings of the 2019 World Wide Web Conference (WWW '19), May 13--17, 2019, San Francisco, CA, USA}
\acmPrice{}
\acmDOI{10.1145/3308558.3313616}
\acmISBN{978-1-4503-6674-8/19/05}

\newtheorem{defn}{Definition}
\newtheorem{thm}{Theorem}
\newtheorem{lem}[thm]{Lemma}

\begin{document}
\title{Policy Gradients for Contextual Recommendations}

\author{Feiyang Pan$^{1,3}$, Qingpeng Cai$^{2}$, Pingzhong Tang$^{2}$, Fuzhen Zhuang$^{1,3}$, Qing He$^{1,3}$}
\affiliation{%
	\institution{\textsuperscript{\rm 1}Key Lab of Intelligent Information Processing of Chinese Academy of Sciences (CAS), \\Institute of Computing Technology, CAS, Beijing, China.\\
	\textsuperscript{\rm 2}IIIS, Tsinghua University.\\
		\textsuperscript{\rm 3}University of Chinese Academy of Sciences, China.
	}
}

\renewcommand{\shortauthors}{F. Pan et al.}
\begin{abstract}
Decision making is a challenging task in online recommender systems. The decision maker often needs to choose a contextual item at each step from a set of candidates. Contextual bandit algorithms have been successfully deployed to such applications, for the trade-off between exploration and exploitation and the state-of-art performance on minimizing online costs. However, the applicability of existing contextual bandit methods is limited by the over-simplified assumptions of the problem, such as assuming a simple form of the reward function or assuming a static environment where the states are not affected by previous actions. 

In this work, we put forward {\em Policy Gradients for Contextual Recommendations (PGCR)} to solve the problem without those unrealistic assumptions. It optimizes over a restricted class of policies where the marginal probability of choosing an item (in expectation of other items) has a simple closed form, and the gradient of the expected return over the policy in this class is in a succinct form. Moreover, PGCR leverages two useful heuristic techniques called Time-Dependent Greed and Actor-Dropout. The former ensures PGCR to be empirically greedy in the limit, and the latter addresses the trade-off between exploration and exploitation by using the policy network with Dropout as a Bayesian approximation. 

PGCR can solve the standard contextual bandits as well as its Markov Decision Process generalization.
Therefore it can be applied to a wide range of realistic settings of recommendations, such as personalized advertising. We evaluate PGCR on toy datasets as well as a real-world dataset of personalized music recommendations. Experiments show that PGCR enables fast convergence and low regret, and outperforms both classic contextual-bandits and vanilla policy gradient methods.
\let\thefootnote\relax\footnotetext{Pingzhong Tang (kenshinping@gmail.com) is the corresponding author.}
\end{abstract}



\maketitle

\section{Introduction}
Decision making in online recommender systems and advertising systems are challenging because the recommender needs to find the policy that maximizes its revenue by interacting with the world. A typical decision-making problem is to select a featured item from a finite set of candidates, for example, to select an advertisement from a set of ads that relate to the user's query in a search engine, or to recommend a song from the user's playlist in a music streaming service. After making the decision, the recommender will receive a reward together with some state transition. In such settings, each item has a so-called \textit{context} (which includes features and attributes) that carries all the necessary information for making the choice. Since it is often the case that the reward, as well as the state dynamic of choosing each item, are related to its context, the recommender system must try to learn how to make the choice given the contexts of all the candidates.

To solve such contextual recommendation problems, algorithms based on \textit{contextual-bandits} have been successfully deployed in a number of industrial level applications over the past decade, such as personalized recommender systems~\cite{li2010contextual,tang2014ensemble,tang2015personalized}, advertisement personalization~\cite{bouneffouf2012contextual,tang2013automatic}, and learning-to-rank~\cite{slivkins2013ranked}. Contextual-bandit algorithms are preferred if one needs to minimize the cumulative cost during online-learning because they aim to address the trade-off between exploitation and exploration. 

The standard contextual bandit problem can be seen as a repeated game between nature and the player~\cite{langford2008epoch}. Nature defines a reward function mapping contexts (a set of features) to real-valued rewards, which is unknown to the player. At each step, nature gives a set of items, each with a context. The player observes the contextual items, selects one, and then receives a reward. The payoff of the player is to minimize the cumulative regret or to maximize the cumulative reward. 

The main challenge of solving contextual bandits lies in the trade-off between exploration and exploitation. The most well-known approaches are arguably value-based methods including Upper Confidence Bounds (UCB) \cite{auer2002finite}, Thompson Sampling (TS) \cite{thompson1933likelihood}, and their variants. These value-based methods try to estimate the expected reward of choosing each item, so they are especially effective when the form of reward function is known explicitly. For example, when the expected reward is linear in the context, \cite{li2010contextual,chu2011contextual,abbasi2011improved} proposed Lin-UCB, which is applied successfully to the online news recommendation of Yahoo, and \cite{chapelle2011empirical,may2012optimistic,agrawal2013thompson} also proposed TS to solve the linear contextual bandits. Similarly, \cite{filippi2010parametric} proposed GLM-UCB using generalized linear models, \cite{krause2011contextual,srinivas2012information} used Gaussian Processes, to model the reward functions. These variants of UCB and TS have been known to achieve sub-linear regrets \cite{auer2002finite,abe2003reinforcement,li2010contextual,chu2011contextual,abbasi2011improved,chapelle2011empirical,agrawal2013thompson,filippi2010parametric,krause2011contextual,srinivas2012information}.
Similar ideas have also been applied to reinforcement learning algorithm such as the UCRL algorithm with regret bounds~\cite{jaksch2010near}.

 However, the applicability of these approaches in real-world applications is heavily limited, especially for large-scale and high-dimensional problems, due to the following reasons: 
\begin{itemize}
    \item First, these methods tend to over-simplify the form of the reward function, which is unrealistic in real-world cases. For example, for sponsored search advertising via real-time bidding, the reward of showing an ad (cost per click) is often the click-through rate multiplied by the bidding price, so it can be understood as a mixture of binary and linear outcomes. Moreover, the reward can often be a high-order non-linear function of features in the contexts.
     \item Second, the overall formulations of contextual-bandit problems are sometimes over-simplified comparing to real-world applications. It is often assumed that the reward is determined by the context of the currently chosen item, and the distribution of contexts is independent of the agent's action. However, it may not be true in real-world recommender systems where the behaviors of users heavily depend on not only the current contexts but the history, i.e., the items that he/she viewed in previous rounds. Also, the set of candidate items can relate to the user's previous preferences as well. These dependencies are not well exploited in existing models.
    \item Last but not least, these methods are value-based, so they are meant to find deterministic policies. A subtle change in the value estimation may cause a discontinuous jump in the resulting policy, which makes convergence difficult for these \cite{sutton2000policy}. On the other hand, stochastic policies are sometimes preferred in online recommender systems.
\end{itemize}

In light of these observations, we propose Policy Gradients for Contextual Recommendations (PGCR), which uses the policy gradient method to solve general contextual recommendations. Our approach model the contextual recommendation problem without unrealistic assumptions or prior knowledge. By optimizing directly over the parameters of stochastic policies, it naturally fits the problems that require randomized actions as well as addresses the trade-off between exploration and exploitation.

Since we design PGCR specifically for contextual recommendations, we would like to specify the performance objective first and see if it is different from the one in standard reinforcement learning. We find that the objective over policies depends on the marginal expected probability of choosing each item (in expectation of other items). So PGCR restricts the search space to a class of policies in which the expected probabilities of choosing an item has a simple closed form and can be estimated efficiently. Therefore, the search space for PGCR is dramatically reduced. 

Then, in order to estimate the marginal probability of choosing each item, we extend Experience Replay, the popular technique in off-policy reinforcement learning \cite{adam2012experience,heess2015memory}, to a finer-grained sampling procedure. By doing so, the variance of estimating policy gradients can be much smaller than the variance of the vanilla policy gradient algorithm. The resulted algorithm is also computationally efficient by stochastic gradient descent with mini-batch training. 

To address the trade-off of exploration and exploitation, our proposed PGCR empirically has the property of Greedy in the Limit with Infinite Exploration (GLIE), which is an essential property for contextual bandits \cite{may2012optimistic}. The property is guaranteed by two useful heuristics named \textit{Time-Dependent Greed}
and \textit{Actor-Dropout}. Time-Dependent Greed is to schedule the level of greed to increase over time, so the resulted stochastic policy will explore a lot in the early stage and then gradually converge to a greedy policy. Actor-Dropout is to use dropout on the policy network while training and inferring, thus the feed-forward network outputs policies with randomness. It has been known that such a stochastic feed-forward neural network can be seen as a Bayesian approximation \cite{gal2016dropout}, so it can provide with directed exploration for PGCR.

Furthermore, with the mentioned techniques, PGCR can directly apply to contextual recommendations in a Markov Decision Process (MDP) setting, i.e. with states and state transitions. We propose this generalized setting for the reason that the i.i.d. assumption on contexts in the standard contextual bandit setting is unrealistic for real-world applications. On the other hand, we suppose that at each step, the contexts are drawn i.i.d. from a distribution conditional on the current state. Furthermore, when an item is chosen, the immediate reward is determined by both the state and the selected item. The state is then transitioned into the next state. Such a model is tailored for a wide range of important realistic applications such as personalized recommender systems where users' preferences are regarded as states and items are regarded as items with contexts \cite{shani2005mdp,taghipour2008hybrid}, and e-commerce where the private information (e.g., cost, reputation) of sellers can be viewed as states and different commercial strategies are regarded as contexts \cite{cai2018reinforcement}.

We evaluate PGCR on toy datasets and a real-world dataset of music recommendation. By comparing with several common baselines including Lin-UCB, GLM-UCB, Thompson Sampling, $\epsilon$-greedy, and vanilla policy gradients, it shows that PGCR converges quickly and achieves the lowest cumulative regret and the highest average reward in various standard contextual-bandits settings. Moreover, when state dynamics are included in the real-world recommendation environments, we find that GLM-UCB and TS fail to incorporate information from the states, while PGCR consistently outperforms other baselines.
\section{Problem formulation}
\subsection{One-step Contextual Recommendations}
We first introduce the simplified setting of contextual recommendation as a standard contextual-bandits problem.
At each step, we have a set of contexts $\mathbf{c} = (c_1,\dots,c_m)^T$ that corresponds to $m$ items, where $c_i$ is the context of the $i^{th}$ item. The contexts $c_1, \dots, c_m$ are independently and identically distributed random variables with outcome space $\mathcal{C}$. 
The action is to select an item from the candidates,
$a\in\{1,...,m\}.$

For the ease of notation, we use $\mathbf{c}$ to denote the concatenation of all $m$ contexts and use $c_a$ to denote the context of the selected item $a$. We write the random variable of immediate reward as $R(c_a)$ to note that in this setting it depends only on the chosen context vector $c_a$. The dependency is not known to the decision-maker. So the target is to learn the dependency and choose the item with the largest expected reward.

A stochastic policy $\pi$ is a function that maps the observations (the set of contexts $\mathbf{c}$) to a distribution of actions. Let the random variable $a\sim\pi(\mathbf{c})$ denote the action determined by policy $\pi$. The performance of a policy is measured as the expected reward of the chosen item over all possible contexts, i.e.,
\begin{equation}
J(\pi) = \mathbb{E}_{\mathbf{c}}\big[R(c_a)\mid a\sim\pi(\mathbf{c})\big],
\label{obj-1}\end{equation}
where $c_a$ is short for the context of the chosen action $a$.

When the policy $\pi$ is parameterized as $\pi_\theta$ where $\theta$ is the trainable parameters, our goal is to find the optimal choice of $\theta$ that maximizes the objective $J(\pi_\theta)$.

However, there is an obvious drawback for this simplified setting: in real-world recommendations, it cannot be assumed that the contexts are always drawn i.i.d from some global probability distribution. For example, when recommending items (goods) to a customer given the searching query in an e-commerce platform, we can only select items from a candidate pool related to the query. This non-i.i.d nature motivates us to put forward a more general setting, which involves the states.
\subsection{Sequential State-aware Contextual Recommendations}
In this part, we introduce the generalized setting as a Markov Decision Process (MDP) with states and state transitions for contextual recommendations, which is referred to as MDP-CR. 

At each step $t$, the decision maker observes its state $s_t$ as well as a set of contexts correlated to that state $\mathbf{c_t} = \{c_{t1}, \dots, c_{tm}\}$. When an action (one of the items) ${a}_t=\pi(s_t,\mathbf{c_t})$ is selected, a reward $R(s_t,c_{ta_t})$ is received, and the state is transitioned to the next state by a Markovian state transition probability $s_{t+1}\sim T(s_{t+1} \mid s_t, c_{ta_t})$. 
Note that the setting in this paper is different from other existing generalized bandits with transitions such as Restless bandits \cite{whittle1988restless}.

In this setting, we assume that the contexts are independently distributed conditioning on the state: $c_{ti}\sim g^{s_{t}}(c)$ for all $i$, where $g^{s_{t}}(c)$ is the probability density of contexts given state $s_{t}$. For example, if the state is the search query or the attribute vector of a user, we assume that the contexts of items in the candidate pool are drawn from a distribution that reflects the search query or the user preference. 

The goal is to find a policy that maximizes the expected cumulative discounted reward, so the objective is 
\begin{equation}
J(\pi) = \mathbb{E}\bigg[\sum_{t=0}^{\infty} \gamma^t R(s_t,c_{ta_t})\mid {a_t}\sim\pi({s_t,\mathbf{c_t}})\bigg],
\end{equation}
where $0<\gamma<1$ is a discount factor that balances short and long term rewards, just like in standard reinforcement learning. We also define the action value function
\begin{equation}Q^{\pi}(s,c)=\mathbb{E}\big[\sum_{t=0}^{\infty} \gamma^{t} R(s_t,c_{ta_t})\mid s_0=s, c_{0a_0}=c, \pi\big].
\label{action-value}\end{equation}

Same as previous works \cite{sutton2000policy,silver2014deterministic} on policy gradients, we denote the discounted state density by $\rho^{\pi}(s)=\int_{\mathcal{S}}\sum_{t=0}^{\infty} \gamma^{t}P_{0}(s_{0})P(s_{0}\rightarrow s, t, \pi)\mathrm{d}s_{0},$
where $P_0(s_0)$ is the probability density of initial states, and $P(s\rightarrow s', t, \pi)$ is the probability density at state $s'$ after transitioning for $t$ time steps from state $s$. 

Thus we can rewrite the objective as
\begin{equation}
J(\pi) = \mathbb{E}_{s\sim\rho^{\pi},\mathbf{c}\sim g^{s},a\sim\pi(s,\mathbf{c})}\big[R(s,c_{a})\big].
\label{J_s_rho}\end{equation}

\section{Policy Gradients for One-Step Contextual Recommendations}
In this section we investigate several key features of our purposed PGCR method. 
For readability, we first discuss the one-step contextual recommendation case (corresponding to section 2.1), which can be modeled as the standard contextual-bandits. Later in the next section, we will show how to extend to the generalized multi-step recommendation which can be modeled as MDP-CR.
\subsection{Marginal probability for choosing an item}

Due to the assumption of the problem setting that the reward only depends on the selected context, we claim that for any policy $\pi$, there exists a \textit{permutation invariant policy} that obtains at least its performance.

\begin{defn}[Permutation invariant policy]
A policy $\pi(\mathbf{c})$ is said to be permutation invariant if for all $\mathbf{c} \in \mathcal{C}^m$ and any its permutation $\mathbf{c}' := P(\mathbf{c})$, it has
\begin{equation}
c'_{a'}  \overset{\text{dist}}{=}  c_{a}
\end{equation}
where $a'\sim\pi(\mathbf{c}')$ and $a\sim \pi(\mathbf{c})$ denote for the actions chosen by the policy $\pi$, $c'_{a'}$ and $c_a$ is the their corresponding contexts respectively, and $\overset{\text{dist}}{=}$ means the probability distribution of the expressions on two sides are the same.
\end{defn}

\begin{lem}
\label{permutaion_in}
For any policy $\pi$, there exists a permutation invariant policy $\pi'$ s.t. $J(\pi')\geq J(\pi)$.
\end{lem}
\begin{proof}
Let us suppose, for the sake of contradiction, that there exists a policy $\pi$ such that

(i) it is not permutation invariant, i.e. there exists $\mathbf{c}\in \mathcal{C}^m$ and some permutation operator $P\in\mathcal{P}$ that $c_{a} \neq c'_{a'}$ where $\mathbf{c}'=P(\mathbf{c})$ and $a'\sim\pi(\mathbf{c'})$;

(ii) The expected reward following $\pi$ is larger than all permutation invariant policies $\tilde\pi$ that $J(\pi) > J(\tilde\pi)$.

Then it follows that
$\mathbb{E}_\mathbf{c}\big[R(\mathbf{c}_{\pi(\mathbf{c})})\big] > \mathbb{E}_\mathbf{c}\big[R(\mathbf{c}_{\tilde\pi(\mathbf{c})})\big]$ for all permutation invariant $\tilde\pi$,
where the expectation is over all sets of contexts. Recall that the contexts are drawn i.i.d. from the same distribution, so we have
\begin{equation}
\mathbb{E}_\mathbf{c}\big[R(\mathbf{c}_{\pi(c)})\big] \equiv \mathbb{E}_\mathbf{c}\bigg[\frac{1}{|\mathcal{P}|}\sum_{P\in \mathcal{P}} R(P(\mathbf{c})_{\pi(P(\mathbf{c}))})\bigg] >
\mathbb{E}_\mathbf{c}\big[R(\mathbf{c}_{\tilde\pi(c)})\big],
\label{confliction_1}\end{equation}
so there exists at least one $\mathbf{c}$ that
\begin{equation}
\frac{1}{|\mathcal{P}|}\sum_{P\in \mathcal{P}} R(P(\mathbf{c})_{\pi(P(\mathbf{c}))}) > R(\mathbf{c}_{\tilde\pi(c)}) \mbox{ for all }\tilde\pi.
\label{confliction_2}\end{equation}
But because $\pi$ is not permutation invariant, we find a policy $\pi^*(P(\mathbf{c}))=\pi((P^*P^TP)(\mathbf{c}))$ that is permutation invariant, where $$P^* = \arg\max_{P\in\mathcal{P}} R(P(\mathbf{c})_{\pi(P(\mathbf{c}))}),$$ then
\begin{equation}
R(\mathbf{c}_{\pi^*(c)})=R(P^*(\mathbf{c})_{\pi(P^*(\mathbf{c}))})>\frac{1}{|\mathcal{P}|}\sum_{P\in \mathcal{P}} R(P(\mathbf{c})_{\pi(P(\mathbf{c}))}),
\end{equation}
which leads to a contradictory to (\ref{confliction_1}) and (\ref{confliction_2}). So it must be that Lemma \ref{permutaion_in} holds.
\end{proof}

Lemma \ref{permutaion_in} states that, without loss of generality, we can focus on permutation invariant policies. The objective in then becomes
\begin{align}
J(\pi) &= \mathbb{E}_{\mathbf{c}}\big[R(c_a)\mid a\sim\pi(\mathbf{c})\big]\notag\\
&= \mathbb{E}_{\mathbf{c}}\bigg[\sum_{i=1}^mR(c_i){I}_{(a=i)} \,\bigg|\, {a}\sim\pi(\mathbf{c})\bigg]\notag\\
&= \sum_{i=1}^m\mathbb{E}_{\mathbf{c}}\bigg[R(c_i){I}_{(a=i)} \,\bigg|\, {a}\sim\pi(\mathbf{c})\bigg]\notag\\
&= \sum_{i=1}^m\,\mathbb{E}_{c_1}\bigg[R(c_1) \mathbb{E}_{c_{-1}}\big[  {I}_{(a=1)} \mid {a}\sim\pi(c,c_{-1})\big]\bigg] \notag\\
&= m\,\mathbb{E}_c\big[R(c)p(c)\big],\label{independent}
\end{align}

where $p(c)$ is the marginal probability of choosing an item with context $c$ (in expectation of randomness of the other $m-1$ items, denoted as $c_{-1}$), by a permutation invariant policy:
\begin{equation}
p(c) = \mathbb{E}_{{c}_{-1}}\big[  {I}_{(a=1)} \mid {a}\sim\pi(c,c_{-1})\big].
\label{marginal_pr}
\end{equation}
Suppose we have a score function $\mu_\theta$ which takes the context as inputs and outputs a score, where $\theta$ are the parameters. We can construct a class $\mathcal{M}$ of permutation invariant policies with the score function:
\begin{equation}
\pi_\theta(\mathbf{c}) \overset{\text{dist}}{=} g\big(\mu_\theta(c_1), \dots, \mu_\theta(c_m)\big),
\label{g-form}
\end{equation}
where $g$ is an operator that satisfies permutation invariance, for example, a family of probability distributions, and $\overset{\text{dist}}{=}$ means the two sides are equivalent in the sense of probability distribution. 

Note that this class of policies include policies of most well-known value-based bandit algorithms. For example, if the score function is the estimation of the reward, and the operator $g(\cdot)$ chooses the item with the maximum estimated reward with probability $1-\epsilon$ and chooses randomly with probability $\epsilon$, the policy is exactly the well-known $\epsilon$-greedy policy \cite{sutton1998reinforcement}. If the score function is a summation of the reward estimation and the upper confidence bound, and $g(\cdot)$ chooses the item with the maximum score, it results in a similar policy to the upper confidence bound (UCB) policy \cite{auer2002finite,li2010contextual}. 

The policy gradient $\nabla_\theta J(\pi_\theta)$ for the standard one-step recommendations can be directly derived from (\ref{independent})
\begin{equation}
\begin{split}
\nabla_\theta J(\pi_\theta) = m\, \mathbb{E}_c \big[R(c)\nabla_\theta p_\theta(c)\big].
\end{split}\label{grad_1}\end{equation}
So it involves computing the marginal probabilities of choosing an item $p(c)$ which is not explicitly parameterized or tractable given an arbitrary policy $\pi_\theta$. To this end, we put forward a restricted family of stochastic policies that allows us to have $p(c)$ in closed-form so as to estimate the gradient of $J(\pi_\theta)$ efficiently. 
\subsection{A simple but powerful class of policies}
Now we propose a class of policies for our PGCR algorithm. For the standard one-step contextual-bandits, following the form of a policy described in (\ref{g-form}), we define a class of stochastic policies denoted by $\mathcal{N}$ as
\begin{equation}
\pi_\theta(\mathbf{c}) = \mbox{Multinoulli}\big\{\sigma\big(\mu_\theta(c_1), \dots, \mu_\theta(c_m)\big)\big\},
\label{mypolicy}\end{equation}
where $\mu_\theta(c)$ is any positive continuous function of $c$ and $\theta$, $\sigma$ is a normalization
$\sigma(\mathbf{x}):=\big(\frac{x_1}{\sum_i x_i}, \dots, \frac{x_m}{\sum_i x_i}\big)$ and Multinoulli$(\cdot)$ returns a multinoulli random variable (the multinoulli distribution is also known as categorical distribution). 

The form of our policy (\ref{mypolicy}) generalizes several important policies in reinforcement learning. For example, when $\mu_\theta(c_i)$ is an exponential function $e^{-\beta f_\theta(c_i)}$, it reduces to the well-known \textit{softmax policy}. If $\beta$ approaches to infinity, it converges to the greedy policy that chooses the item with highest score.

For any policy $\pi_{\theta}\in \mathcal{N}$, the marginal probability of choosing an item can be easily derived as 
\begin{equation}
p_\theta(c_i)= \mathbb{E}_{{c}_{-i}}\bigg[\frac{\mu_\theta(c_i)}{\mu_\theta(c_i)+\sum_{j\neq i}\mu_\theta(c_j)}\bigg], 
\end{equation}which is a continuous and positive function of parameters $\theta$. So it is straightforward to estimate $p_\theta(c)$ by sampling ${c}_{-i}$ from the contexts that the player have seen before. We denote by $D_t$ the contexts that have appeared up to step $t$. The estimation of $p(c)$ is unbiased because as assumed all contexts in $D_t$ are i.i.d from the context space.

\section{Policy Gradients for Sequential Contextual Recommendations}
In this section we extend this approach to MDP-CR. We first use $\tilde{c}$ to denote the \it augmented context \rm by pairing together a state $s$ and the context $c$ for a certain item, i.e., $\tilde{c} =(s,c)$. For example, in a personalized recommender system, the state $s$ consists of the user feature, $c$ contains the item feature. Their concatenation $\tilde{c} =(s,c)$ is often used in ordinary tasks such as the prediction of click-through rate. 

Given a policy $\pi$, the states can be roughly thought of as drawn from the discounted stationary distribution $\rho^{\pi}(s)$. As we already defined the density of contexts given the state as $g^s(c)$, we have the discounted density of the augmented context $\tilde{c}$ by the axiom of probability 
$$\xi^{\pi}(\tilde{c}) = \rho^{\pi}(s)\,g^{s}(c).$$ Since we assume the state distribution $\rho^{\pi}(s)$ is stationary, it is natural that $\xi^{\pi}(\tilde c)$ is also stationary.

Then by applying the same technique as we derive the marginal probability, we derive the performance objective as follows:
\begin{equation}
J(\pi) = m\,\mathbb{E}_{\tilde{c}\sim \xi^{\pi}}\big[R(\tilde{c})\cdot p(\tilde{c})\big],
\end{equation}
where $
p(\tilde{c})=\mathbb{E}_{c_{-1}\sim g^{s}}[{I}(a=1)\mid \pi]
$ is the marginal probability of choosing $c$ given $s$. 

Now we derive the gradients of $J(\pi)$ similar to the result in \cite{sutton2000policy}. Surprisingly, it only replaces $R(c)$ in (\ref{grad_1}) by $Q(\tilde{c})$. 
\begin{thm} Assuming the policy $\pi$ leads to stationary distributions for states and contexts, the unbiased policy gradient is
\begin{equation}
\nabla_{\theta}J(\pi_{\theta})=m\,\mathbb{E}_{\tilde c \sim \xi}\big[\nabla_\theta p_{\theta}(\tilde{c}) \,Q(\tilde{c})\big],
\label{pg-thm}\end{equation}
where $Q^\pi(\tilde c):=Q^\pi(s, c)$ as defined in (\ref{action-value}), and $\xi^\pi(\tilde{c})$ is the discounted density of $\tilde c$.
\label{thm_of_pg}\end{thm}

\begin{proof}
We denote the state-value for state $s$ under policy $\pi$ as
\begin{equation*}
V^{\pi}(s)=m\int_c p_{\theta}(s,c) Q^{\pi}(s, c)g^s(c)\mbox{d}c,
\label{value function}
\end{equation*}
it follows that
\begin{align}
\nabla_\theta V^\pi(s)
&=\nabla_\theta m\int_c p_{\theta}(s,c) Q^{\pi}(s, c)g^s(c)\mbox{d}c\notag\\
&=m\int_c \nabla_\theta p_{\theta}(s,c) Q^{\pi}(s, c)g^s(c)\mbox{d}c\notag\\
 &\quad+\gamma m \int_{s'} P(s\rightarrow s',1,\pi) \nabla_\theta V^\pi(s')\mbox{d}s',\notag
\end{align}
By repeatedly unrolling the equation, we have
\begin{equation*}
\nabla_\theta V^\pi(s) = m\int_{s'}\sum_{t=0}^{\infty}\gamma^t P(s\rightarrow s', t, \pi)\int_c \nabla_\theta p_{\theta}(\tilde{c})Q(\tilde{c})g^{s'}(c)\mbox{d}c\mbox{d}s^{'}.
\end{equation*}
Integrating both side over the start-state and recalling the discounted state density $\rho^\pi(s)$ and discounted augmented context density $\xi^\pi(\tilde c)$, we get the policy gradient as
\begin{align}
\nabla_\theta J(\pi) &= m\int_s P_0(s) \nabla_\theta V^\pi(s) ds\notag\\
&= m\int_{s}\rho^\pi(s)\int_{c} \nabla_\theta p_{\theta}(\tilde{c})Q(\tilde{c})g^{s}(c)\,\mbox{d}c\mbox{d}s\notag\\
&= m\int_{\tilde{c}}\nabla_\theta p_{\theta}(\tilde{c}) Q(\tilde{c})\xi^\pi(\tilde{c})\mbox{d}\tilde{c}.\notag
\end{align}
\end{proof}

Again we can find the restricted class of policies to perform efficient estimation of the gradient. For MDP-CR, it is straightforward to extend (\ref{mypolicy}) to introduce states by replacing $c$ as the augmented context $\tilde{c}$.

\section{Actor-Critic with Function Approximations}
In this section, we deliver the practical algorithm to estimate the policy gradient using function approximations. In the sample-based approximations, we write collected the reward feedback as $r(c)$ for the one-step recommendation case, and $r(\tilde c)$ for the sequential recommendation case, as the realizations of the reward random variable $R(\cdot)$. 

In conventional methods based on contextual-bandits, the most direct way to estimate the reward function is to directly apply supervised learning methods to find an estimator $f_\phi$ with parameter $\phi$ minimizing the mean squared error, i.e.,
\begin{equation}
\min_\phi \frac{1}{\vert D_t^{(1)} \vert}\sum_{c\in D_t^{(1)}}\big(r(c) - f_\phi(c)\big)^2,
\label{DM_min}\end{equation}
where $D_t^{(1)}\subset D_t$ is the set of chosen contexts and $r(c)$ is the received reward for choosing context $c$. It is common for most value-based contextual bandit algorithms, such as $\epsilon$-greedy, Lin-UCB and Thompson Sampling. 

However, we argue that in a policy-based perspective, this kind of off-policy supervised learning brings bias. Since our goal is to maximize the expected reward $J(\pi_\theta)$ rather than minimizing the empirical loss as in supervised learning, the marginal probabilities of choosing an item must be taken into consideration, and the form of $f_\phi(c)$ cannot be chosen arbitrarily. 

Similarly, when states and state transitions are involved in MDP-CR, we also need to find an approriate $f_\phi(\tilde{c})$ to approximate $Q(\tilde{c})$. 
Since we have noted that the standard one-step contextual-bandits can be seen as a special case of the generalized MDP-CR, from now we will use the notations of MDP-CR by default to deliver the main results and algorithms.

We take similar spirit to \cite{sutton2000policy,silver2014deterministic} and define the following \it compatible \rm conditions, to assure that the policy gradient is orthogonal to the error in value approximation. 
\begin{thm}
\label{Compatible}
The policy gradient using function approximation
\begin{equation}
\nabla_{\theta}J(\pi_{\theta}) = m\int_{\tilde{c}} \nabla_{\theta}p_{\theta}(\tilde{c})\cdot f_{\phi}(\tilde{c})\xi^\pi(\tilde{c})\mbox{d}\tilde{c}
\label{comp-pg}\end{equation}
is unbiased to (\ref{pg-thm}) if the following conditions are satisfied:

(i) the gradients for the value function and the policy are compatible,
\begin{equation}\nabla_{\phi}f_{\phi}(\tilde c)= \nabla_{\theta}\,\mbox{log}\,p_{\theta}(\tilde c),
\end{equation}

(ii) the value function parameters $\phi$ reach a local minimum of the mean squared error over the stationary context distribution such that
\begin{equation}\nabla_\phi\,\mathbb{E}_{\tilde c \sim \xi^\pi}\big[p_{\theta}(\tilde c)\big(f_{\phi}(\tilde c) - Q^{\pi}(\tilde c)\big)^{2}\big]=0.
\end{equation}
\end{thm}

\begin{proof}
By condition (ii), as we assumed the distribution of contexts $\xi^\pi$ is stationary with respect to the policy $\pi$, it is easy to see when the conditions hold,
\begin{equation*}
m\int_{\tilde{c}}\xi^\pi(\tilde c)p_{\theta}(\tilde c)[Q^{\pi}(\tilde{c})-f_{\phi}(\tilde{c})\big]\nabla_{\phi}f_{\phi}(\tilde c)=0.
\end{equation*}
Then by condition (i) we have
\begin{equation*}
m\int_{\tilde{c}} \nabla_{\theta}p_{\theta}(\tilde{c})\big[Q^{\pi}(\tilde{c})-f_{\phi}(\tilde{c})]\xi^\pi(\tilde{c})\mbox{d}\tilde{c}=0,
\end{equation*}
which is the difference between (\ref{pg-thm}) and (\ref{comp-pg}).
\end{proof}

\subsection{The basic PGCR algorithm}
We now formally propose the policy gradients algorithm for general contextual recommendations, coined by PGCR. Recall that our policy returns a Multinoulli random variable that chooses $a_t$ by
\begin{equation*}
a_t\sim \mathrm{Multinoulli}\big\{\sigma\big(\mu_\theta(s_t,c_{t1}), \dots, \mu_\theta(s_t,c_{tm})\big)\big\}.
\end{equation*}
The key is to estimate the marginal expected probabilities for each item. When estimating it for some context, say $c_{ti}$, at some state $s_t$, we resample from all the previously observed contexts at the same state to get another $m-1$ contexts.
\begin{equation}
\hat{p}_\theta(s_t, c_{ti})= \frac{1}{N}\sum_n^N \frac{\mu_\theta(s_t, c_{ti})}{\mu_\theta(s_t, c_{ti})+\sum_c \mu_\theta(s_t, c)}
\end{equation}
where $N$ ($N\geq 1$) is the number of resampling times, and $c$ in the denominator are another $m-1$ sampled contexts from the same state to $s_t$. 

Similar to previous actor-critic algorithms \cite{lillicrap2015continuous}, we can use Sarsa updates \cite{sutton1998reinforcement} to estimate the action-value function and then update the policy parameters respectively by the following \it policy gradients for contextual-bandits \rm algorithm,
\begin{align}
&\delta_t=r_t + \gamma f_{\phi_t}(s_{t+1}, c_{(t+1)a}) - f_{\phi_t}(s_t, c_{ta})\\
&\Delta_{\phi,t}^{\mathrm{PGCR}}= \hat{p}_\theta (s_t, c_{ta})\delta_t \nabla_\phi f_{\phi_t}(s_t, c_{ta})\\
&\phi_{t+1}=\phi_{t} + \alpha_\phi \Delta_{\phi,t}^{\mathrm{PGCR}}\\
&\Delta_{\theta,t}^{\mathrm{PGCR}}=\sum_{i=1}^m \nabla_\theta \hat{p}_{\theta_t} (s_t, c_{ti}) f_{\phi_{t+1}}(s_t, c_{ti})\label{update_PGCR}\\
&\theta_{t+1}=\theta_t + \alpha_\theta \Delta_{\theta,t}^{\mathrm{PGCR}}.
\end{align}
In practice, the gradients can be updated on mini-batches by modern optimizers such as the Adam optimizer \cite{kingma2014adam} which we already used for experiments. PGCR naturally fits to deep Reinforcement Learning and Online Learning, and techniques from these area may also be applied. 
Note that the algorithm can also be apply to the standard contextual bandit setting without states. 
\subsection{Two useful heuristics}
\subsubsection{Time-Dependent Greed} 

Greedy in the Limit with Infinite Exploration (GLIE), is the basic criteria desired for bandit algorithms. GLIE is to explore all the actions infinite times and then to converge to a greedy policy that reaches the global optimal reward if it runs for enough time. Value-based methods can satisfy GLIE if a positive but diminishing \it exploration value \rm is given to all the actions. But for policy-based methods, it is not straightforward, because one cannot explicitly show the exploration level of a stochastic policy. 

For PGCR, on the contrary, it is easy to have GLIE by \it Time-Dependent Greed\rm, which applies a Time-Dependent Greed factor to the scoring function $\mu(c)$. A straightforward usage is to let $\mu_\textrm{greedy}(c;t):=\mu^{\alpha t}(c)$ where $\alpha$ is a pre-determined positive constant value and $t$ is the current time-step. When $t\to\infty$, the policy tends to choose only the item with the largest score. Also the marginal probability $p_\textrm{greedy}(c;t)$ remains positive with the assumption that $\mu(c)\in (0, +\infty)$ for all context $c$, so any item gets an infinite chance to be explored if it runs for enough time. This technique can also apply to other policy-based RL methods as well.

\subsubsection{Actor-Dropout} 

Directed exploration is also desired. UCB and TS methods are well-known to have directed exploration so that they automatically trade-offs between exploration and exploitation and get sub-linear total regrets. The basic insight of UCB and TS is to learn the model uncertainty during the online decision-making process and to explore the items with larger uncertainty (or higher potential to get a large reward). Often a Bayesian framework is used to model the uncertainty by estimating the posterior. However, the limitation of these methods is that assumptions and prior knowledge of reward functions are required, otherwise the posterior cannot be estimated correctly. 

In light of these observations, we propose \textit{Actor-Dropout} for PGCR to achieve directed exploration. The method is simple: to use dropout on the policy network while training and inferring. We do so because it has been theoretically and empirically justified that a neural network with dropout is mathematically equivalent to a Bayesian approximation \cite{gal2016dropout}. So Actor-Dropout naturally learns the uncertainty of policies and does Monte Carlo sampling when making decisions. Since Actor-Dropout needs no prior knowledge, it can apply to more general and complex cases than UCB and TS.

To use Actor-Dropout, in practice it is good enough for exploration to add dropout to just one layer of weights. For example, for a fully-connected actor-network, one can use dropout to the weights before the output layer, with a dropout ratio of 0.5 or 0.67. It can be understood as to train several actors and to randomly pick one at each step, so it trade-offs between exploration and exploitation since each actor learns something different from each other. We also found Actor-Dropout worth trying for other RL or Online Learning tasks in the exploration phase.

\subsection{Lower variance of the gradients of PGCR than vanilla PG}
We prove that the variance of updating the actor and the critic of PGCR is less than that of vanilla PG.

Since the concept of \textit{context} does not exist in the classic formulation of reinforcement learning, it is often regarded as part of the state. Given a stochastic policy $\pi_\theta(s, \mathbf{c})$, PG has policy gradients
\begin{equation}
\nabla_{\theta}J(\pi_{\theta}) = m\sum_s \rho_s^{\pi}\sum_{i=1}^m \nabla_{\theta}[e_i^T\pi_\theta(s, \mathbf{c})]\cdot f_{\phi}(\tilde{c}_i),
\label{grad_pg}\end{equation}
where $e_i$ denotes a unit vector and $e_i^T \pi_\theta(s, \mathbf{c})$ is the probability for choosing the $i^{\mathrm{th}}$ item. For simplicity, we write $ \nu_i:=e_i^T \pi_\theta(s, \mathbf{c})$. Since we focus on policy gradients, we assume that PG has a critic function $f_\phi(\tilde c)$ with the same form as PGCR. The corresponding update steps for PG is
\begin{align}
&\Delta_{\phi_t}^{\mathrm{PG}}= \nu_{ta}\delta_t \nabla_\phi f_\phi(s_t, c_{ta})\\
&\Delta_{\theta_t}^{\mathrm{PG}}= \sum_i^m \nabla_\theta \nu_{ti} f_{\phi_{t+1}}(s_t, c_{ti}).
\label{update_pg}\end{align}

Our PGCR can achieve lower estimation variance comparing to classic stochastic policy gradient methods such as \cite{sutton2000policy}, the reasons are two-fold. Firstly, by Lemma \ref{permutaion_in} we know permutation invariant policies are sufficient for contextual-bandits problems, so PGCR adopts the restricted class of policies. On the contrary, in vanilla policy gradients, one should treat a state $s$ and the whole contexts $\mathbf{c}$ altogether as the input of the policy function, so the sample space can be much larger, which results in lower sample efficiency. Secondly, even if with the same form of policy, vanilla policy gradients tend to converge slower than PGCR because they do not take the marginal expected probabilities of choosing an item into consideration.

In a formal way, we can have the following conclusion.

\begin{lem} Given a policy $\pi_\theta\in\mathcal{N}$ and a value approximation $f_\phi$, both $\Delta_{\phi_{t}}^{PGCR}$ and $\Delta_{\phi_{t}}^{PG}$ are unbiased estimators for the true gradients of action-value approximation
\begin{equation}
\Delta_{\phi_{t}} = p(s_t, c_{ta}) \delta_t \nabla_\phi f_\phi(s_t, c_{ta}).
\end{equation} And $\mbox{Var}\big[\Delta_{\phi_{t}}^{PGCR}\big] \leq \mbox{Var}\big[\Delta_{\phi_{t}}^{PG}\big]$. Additionally if PGCR uses a fixed $N$, as $t\rightarrow +\infty$, with probability $1$ we have
\begin{equation}
\mbox{Var}\big[\Delta_{\phi_{t}}^{PGCR}\big] \rightarrow \frac{1}{N}\mbox{Var}\big[\Delta_{\phi_{t}}^{PG}\big].
\end{equation}
\label{variance}\end{lem}
\begin{proof}It is obvious that both $\nu_{ta}$ in $\Delta_{\phi_{t}}^{PG}$ and $\hat{p}(s_t, c_{ta})$ in $\Delta_{\phi_{t}}^{PGCR}$ are unbiased to $p(s_t, c_{ta})$. So both $\Delta_{\phi_{t}}^{PGCR}$ and $\Delta_{\phi_{t}}^{PG}$ are unbiased to $\Delta_{\phi_{t}}$.
To analyze the variance, we focus on the estimations of the probability of choosing an item: $\nu_{ta}$ and $\hat{p}(s_t, c_{ta})$. Let $V:=\mbox{Var}\big[\nu_{ta}\big].$ Then for PGCR,
\begin{equation}
\mbox{Var}\big[\hat{p}(s_t, c_{ta})\big]=\mbox{Var}\bigg[\frac{1}{N}\sum_{n=1}^N \nu^{(n)}_{ta}\bigg],
\end{equation}
where $\nu^{(n)}_{ta}$ denotes the probability of choosing $c_{ta}$ at the $n^\mathrm{th}$ time of sampling. In the worst case, it samples exactly the same set of $m-1$ items every time, then $\mbox{Var}\big[\hat{p}(s_t, c_{ta})\big]=V$. Otherwise if there exists $n_1$ and $n_2$ that the samples are different such that $\nu_{ta}^{(n_1)}\neq \nu_{ta}^{(n_2)}$, then the correlation is strictly less than $1$ and we have $\mbox{Var}\big[\Delta_{\phi_{t}}^{PGCR}\big] < \mbox{Var}\big[\Delta_{\phi_{t}}^{PG}\big]$ in this case. Finally when enough time steps passed, for $N$ is a fixed positive integer, the probability of each item being sampled at most once is
\begin{equation*}
{mt\choose (m-1)N}(mt)^{-(m-1)N} \rightarrow 1\quad \mbox{as}\quad t\rightarrow +\infty.
\end{equation*}
So with probability $1$ the sampled contexts are all different to each other so the estimated probabilities of choosing an item are i.i.d., then
$\mbox{Var}\big[\Delta_{\phi_{t}}^{PGCR}\big] \rightarrow V/N$.
\end{proof}
We get the following theorem applying the similar technique to Lemma \ref{variance}. We claim that Policy gradients (\ref{comp-pg}) has no higher variance than gradients in PG.

\begin{thm}$
\mbox{Var}\big[\Delta_\theta^{\mathrm{PGCR}}\big]\leq \mbox{Var}\big[\Delta_\theta^{\mathrm{PG}}\big]
$.\label{var_pg}\end{thm}

\begin{proof}
Similar to the proof of Lemma \ref{variance}, we denote the variance of $\Delta_\theta^{\mathrm{PG}}$ by $V_\theta$.
\begin{equation}
V_\theta := \mbox{Var}\big[\Delta_\theta^{\mathrm{PG}}\big]=\mbox{Var}\bigg[\sum_i^m \nabla_\theta \nu_{ti} f_{\phi_{t+1}}(s_t, c_{ti})\bigg]
\end{equation}
By the update rules (\ref{update_PGCR}) of PGCR,
the variance of $\Delta_\theta^{\mathrm{PGCR}}$ is
\begin{equation*}\begin{split}
\mbox{Var}\big[\Delta_{\theta_t}^{\mathrm{PGCR}} \big]
&= \mbox{Var}\bigg[\sum_{i=1}^m \nabla_\theta \hat{p}_\theta (s_t, c_{ti}) f_{\phi_{t+1}}(s_t, c_{ti})\bigg]\\
&= \mbox{Var}\bigg[\frac{1}{N}  \sum_{n=1}^N\sum_{i=1}^m \nabla_\theta \nu^{(n)}_{ti} f_{\phi_{t+1}}(s_t, c_{ti})\bigg]\\
&\leq \mbox{Var}\bigg[ \sum_{i=1}^m \nabla_\theta \nu^{(n)}_{ti} f_{\phi_{t+1}}(s_t, c_{ti})\bigg],\quad \forall\,n=1,\dots,N.
\end{split}\end{equation*}
Because of the assumption that the sampled contexts in each sampling procedure are independent and identical distributed, we have $\mbox{Var}\big[ \sum_{i=1}^m \nabla_\theta \nu^{(n)}_{ti} f_{\phi_{t+1}}(s_t, c_{ti})\big] = V_\theta$ for all $n=1,\dots,N$ and the theorem is proved.
\end{proof}
Note that, in practice, PGCR does not necessarily set $N$ to a large integer since it is naturally a finer-grained experience replay \cite{adam2012experience}. Surprisingly, when $N=1$, PGCR can have a better performance than $PG$ even in the simplest setting. In the next section, we will demonstrate experimental results that show that PGCR with $N=1$ achieves better performance in various settings compared to other baseline methods including PG.

The results can be interpreted as follows. From a statistical point of view, PGCR takes advantage from a resampling technique so the estimations have lower variances. From an optimization perspective, PGCR reduces the correlation of estimating probabilities of choosing the $m$ items within the same time step, so it has less chance to suffer from exploiting and over-fitting, while vanilla PG cannot. For example, when the estimated values of $m$ contexts are given, an optimizer for PG would simultaneously increase one item's chosen probability and reduce other $m-1$ ones', which results in training the policy into a deterministic one: the item with the largest estimated value will get a chosen probability close to $1$, and others get arbitrary small probabilities close to $0$. Afterward, the items with $0$ chosen probabilities will hardly have any influence to further updates. So eventually, vanilla PG is likely to over-fit the existing data. On the contrary, when PGCR estimates the gradients, even if an item cannot beat against the other $m-1$ competitors at its own time step, it can still help because it outranked some items from other time steps. Therefore, PGCR tends to be more robust and explores better than vanilla PG.

\begin{figure*}[!h]
    \centering
    \begin{minipage}{.5\textwidth}
        \centering
        \includegraphics[width=\linewidth]{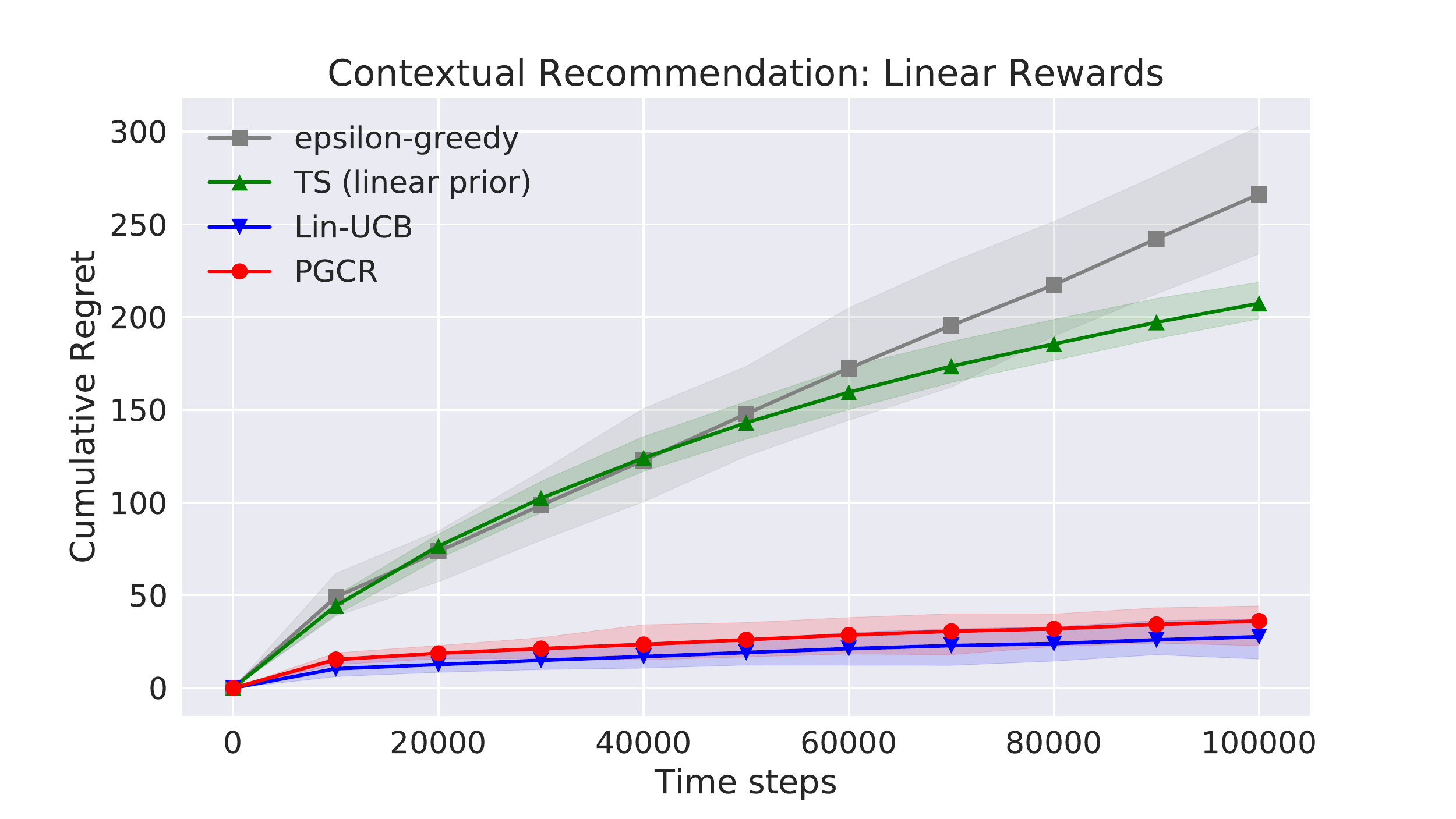}
        (a)
    \end{minipage}
    \begin{minipage}{0.5\textwidth}
        \centering
        \includegraphics[width=\linewidth]{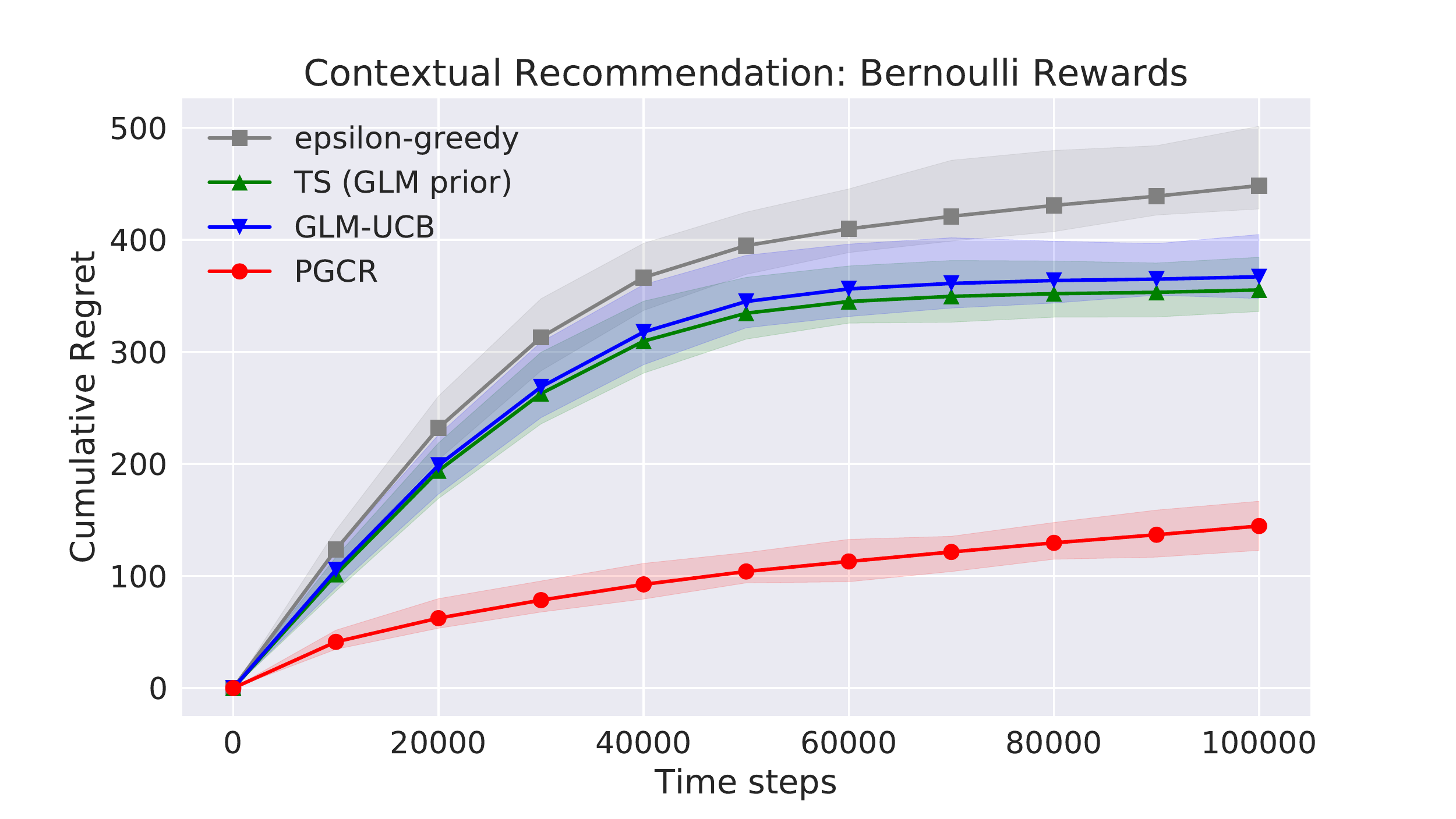}
        (b)
    \end{minipage}\\
    
    \begin{minipage}{0.5\textwidth}
        \centering
        \includegraphics[width=\linewidth]{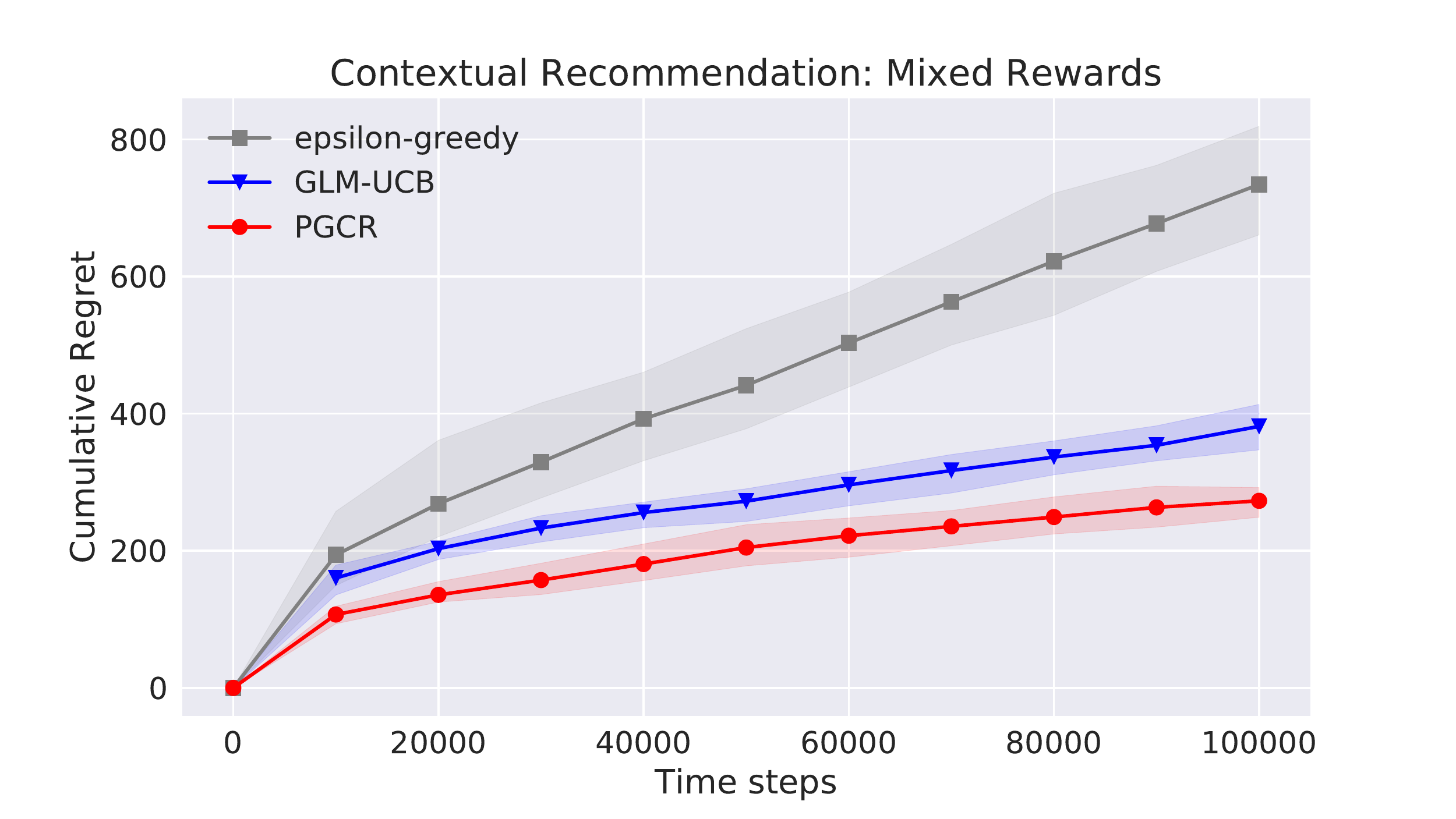}
        (c)
    \end{minipage}
    \begin{minipage}{0.5\textwidth}
        \centering
        \includegraphics[width=\linewidth]{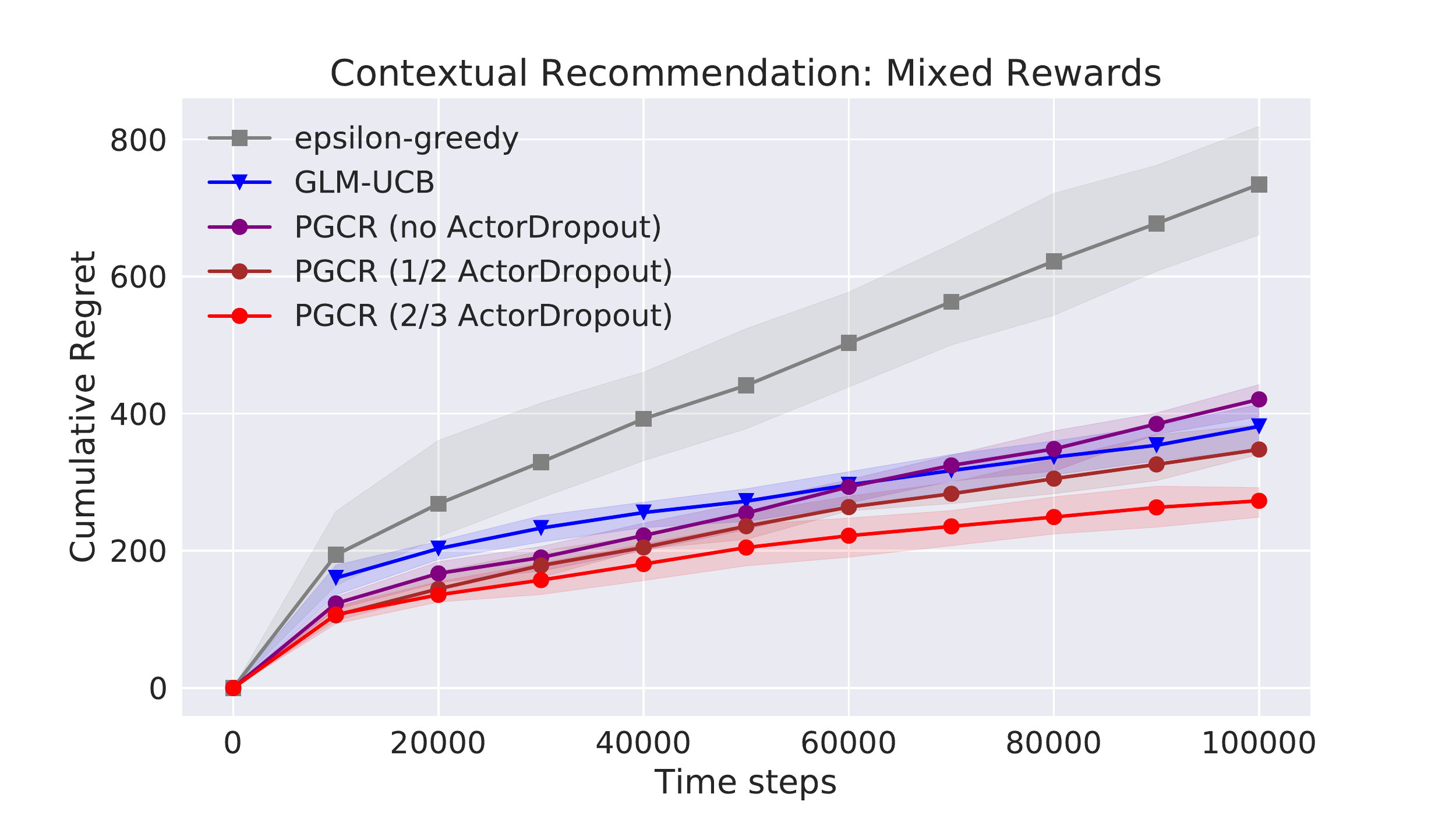}
        (d)
    \end{minipage}
    \caption{Experiments on toy data. The solid lines are averaged cumulative regrets over 20 runs, and the shaded areas stand for the standard deviations. (a) Linear rewards: PGCR perform comparably to Lin-UCB and is much better than the other two. (b) Bernoulli rewards: PGCR outperforms the others with large margins. (c) Mixed-rewards: PGCR outperforms the others. (d) Testing PGCR with different level of Actor-Dropout. It shows that, without Actor-Dropout, PGCR tends to have linear regret like $\epsilon$-greedy and fail to converge. But when using Actor-Dropout, PGCR empirically converges which is similar to UCB, yet gets even lower regret than UCB.}\label{toy}
\end{figure*}
\section{Experiments}
\subsection{Datasets and simulation details}
We test our proposed PGCR and other baseline methods on several simulated contextual recommendation environments. To start with, we would like to describe the details of our simulations.

\begin{itemize}
\item \textbf{Toy environments:} The first set of experiments are done with a generated toy dataset to simulate the standard contextual bandit settings. We simulate a contextual bandit environment with $5$ items at each step, where each item is represented by a $40$-dimensional context vector that is i.i.d. sampled from a uniform distribution in a unit cube $c\sim \mathrm{U}(\mathcal{C})$, $\mathcal{C}={(0, 1)}^{40}$. Once an item with context $c$ is chosen by the player, the environment returns a random reward $R(c)$. We test three types of reward functions: (a) the linear reward with Gaussian noise, as $R(c):=w_{r}^{T}c+e_{r}$; (b) the Bernoulli reward, as $R(c)\sim \textrm{Bernoulli}(\beta(c))$ where $\beta(c):=w_{\beta}^{T}c+e_{\beta}\in[0,1]$ is the probability to return reward $1$ for choosing $c$; (c) the mixed reward, which first returns a random linear reward $w_{r}^{T}c+e_{r}$ with probability $\beta(c)$ and returns a zero reward with probability $1-\beta(c)$, as a mixture of binary and linear rewards. $w_{r}$ and $w_{\beta}$ are coefficients unknown to the player. $e_{r}$ and $e_{\beta}$ are white noises to introduce some randomness. 

\item \textbf{Music recommendation environments:} We use a real-world dataset of music recommendation provided by \it KKBox \rm and open-sourced on \it Kaggle.com\rm\footnote{https://www.kaggle.com/c/kkbox-music-recommendation-challenge/}. The challenge is to predict the chances of a user listening to a song repeatedly. So we construct simulation environments to simulate the online contextual recommendations to test our methods.

We construct two simulators based on the distributions of the dataset with different settings: one simplified setting without explicit states and one general setting with states and state transitions. At each time step, a user comes to the system, who is randomly picked from the users in the dataset. For the state-aware setting, we set the last 3 songs the system recommended previously to this user together with the corresponding feedbacks (listened or not) as his/her current state. Then the simulator randomly samples a set of $10$ candidate songs from the user's listening history. Finally, the recommender needs to make a decision to recommends one to the user. If the user listens to it again (which is the original target for the supervised learning dataset), the system gets a reward $1$, otherwise, it gets a reward $0$. Each song has 94 fields of features in the context vector, including discrete attributes and numerical features about the song's genre, artists, composers, language, etc. The simulation consists of 5 million time steps and is repeated for 5 runs. Since the optimal policy is unknown in this problem, we will use the average reward as the performance metric.

\end{itemize}
\subsection{Experiments on toy datasets}
For the three standard contextual-bandit problems, we use the cumulative regret as the evaluation metric. The cumulative regret is defined as the cumulative difference between the received reward and the reward of the optimal item.

For PGCR, we use the multi-layer perceptron (MLP) as the policy and value networks. The MLPs have one fully-connected hidden layer with $10$ units and the ReLU activation. As for the training details, at each step, we sample a batch of 64 history samples for efficient experience replay. The loss is minimized by the gradient-based Adam optimizer \cite{kingma2014adam}. 

We compare PGCR with the following algorithms: 
\begin{itemize}
\item \textbf{$\epsilon$-greedy}: It chooses the item with the largest estimated value with a probability of $1-\epsilon$ and chooses randomly otherwise. Specifically, it estimates the expected reward by a value network with the same structure of PGCR. 
\item \textbf{Lin-UCB}: The widely studied version of UCB for contextual bandits with linear rewards \cite{li2010contextual,chu2011contextual,abbasi2011improved}, which uses a linear function to approximate the reward, and chooses the item with the maximum sum of the estimated reward and the estimated confidence bound. 
\item \textbf{GLM-UCB}: The UCB method for generalized linear rewards, proposed in \cite{filippi2010parametric}, which can solve non-linear rewards if the reward function can be fitted by a generalized linear model of contexts, such as Bernoulli rewards and Poisson rewards. 
\item \textbf{Thompson Sampling}: It samples from the posterior distribution of parameters, estimates each item's value, and chooses the item with the maximum estimation \cite{chapelle2011empirical,agrawal2013thompson}. Specifically, it uses the same function approximation as Lin-UCB or GLM-UCB for linear and non-linear rewards. 
\end{itemize}
For Lin-UCB, GLM-UCB and Thompson Sampling, we use the same training procedures as suggested in \cite{li2010contextual}. For $\epsilon$-greedy, $\epsilon$ is set constantly to $0.1$. 

For each method, in order to reduce the randomness of experiments, we run the simulation for 20 times and report their averaged cumulative regrets. The results are shown in Figure \ref{toy}. 

From Figure \ref{toy}(a), (b) and (c), it sees that PGCR converges more quickly and has lower regrets in most cases. For the linear bandits, Lin-UCB is, theoretically, one of the best choice if the linearity is known. Even though PGCR does not know the form of reward function apriori, the performance is comparable with Lin-UCB. Thompson Sampling converges slower and $\epsilon$-greedy obviously fails to converge. For the generalized linear case, PGCR achieves much lower regret than other baselines. For the mixed reward case, we see that PGCR learns faster than GLM-UCB and empirically converges after a long run.

We further did an ablation study to test the improvements induced by the proposed heuristic, Actor-Dropout. See Figure \ref{toy}(d) for the comparisons. 

It shows that Actor-Dropout significantly helps PGCR to converge. The growth rate of cumulative regret for PGCR without Actor-Dropout is similar to $\epsilon$-greedy, indicating that the original algorithm fails to converge and has linear regrets. But when equipped with Actor-Dropout, the regrets are smaller. When the dropout rate is set to $0.67$, the growth rate of PGCR's regret is similar to that of GLM-UCB which can theoretically achieve sub-linear regrets. So empirically we remark that Actor-Dropout is a strong weapon for PGCR in order to get a convergence guarantee, even with almost no assumptions on the problem.
\subsection{Experiments on music recommendations}
\subsubsection{Simplified setting without explicit states}

The experimental setup in the setting without states is as follows: PGCR uses MLPs as policy and value networks. Each MLP has two hidden layers of sizes 60 and 20 respectively and uses the ReLU nonlinearity. $\epsilon$-greedy has exactly the same value network like the one in PGCR. All these methods are trained with Adam algorithm with the same learning rates. The batch-size is set as 256. We also test GLM-UCB here as another baseline because the reward is binary thus can be learned by a logistic regression model.

 As is shown in Figure \ref{KKBox}(a), PGCR performs the best. It is interesting to see that traditional contextual-bandits methods learn fast from the beginning, which indicates that they are good at exploration, but their average rewards stop increasing rapidly due to the limitation of the fitting power of general linear models. $\epsilon$-greedy learns slowly at the early stage, but it outperforms GLM-UCB and TS after a long run. Comparing with these algorithms, PGCR has the best performance from the beginning to the end of the learning process.
\begin{figure}[!htb]
    \centering
    \begin{minipage}{.45\textwidth}
        \centering
        \includegraphics[width=\linewidth]{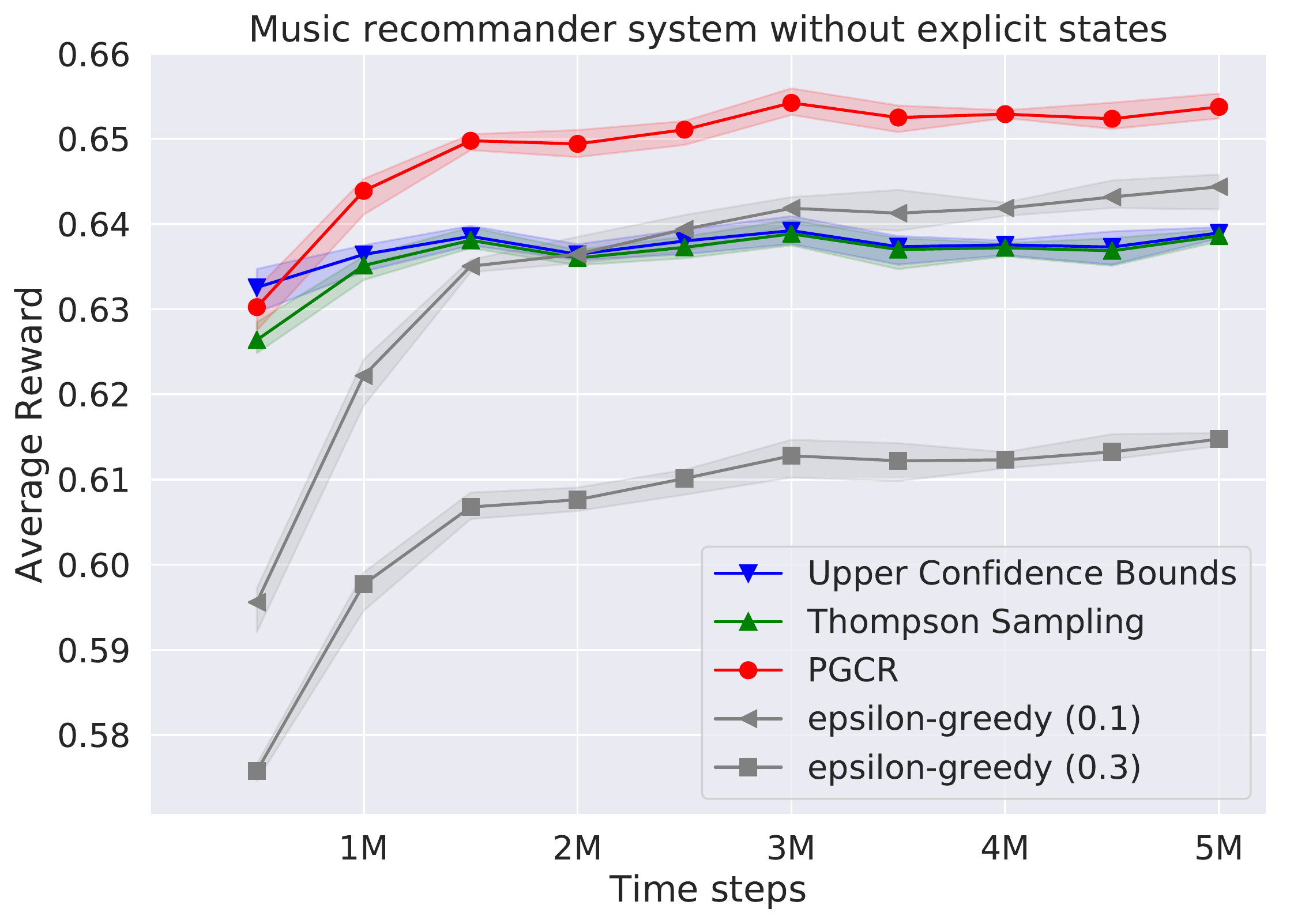}
    \end{minipage}\\
    (a)\\
    \begin{minipage}{.45\textwidth}
        \centering
        \includegraphics[width=\linewidth]{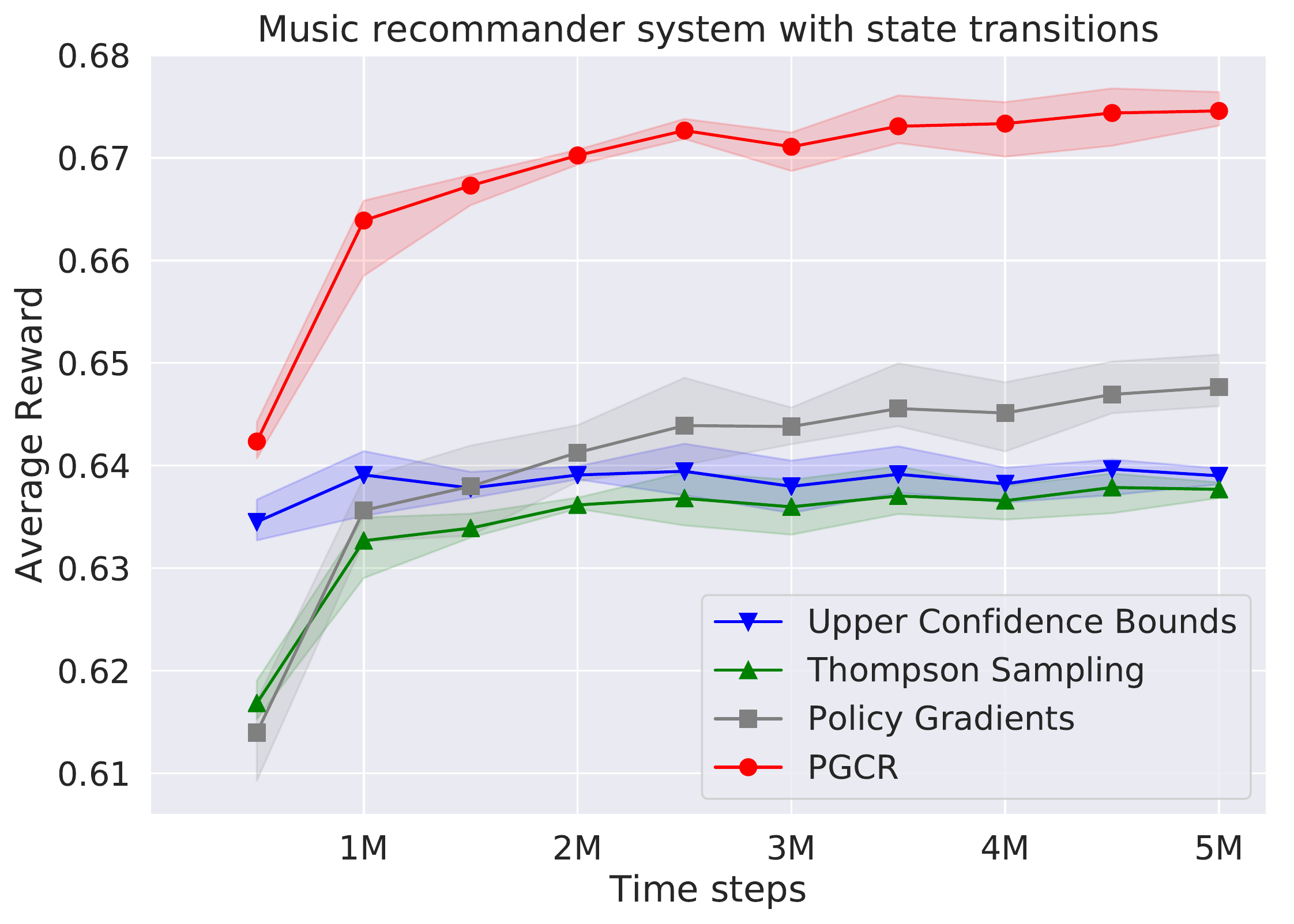}
    \end{minipage}\\
    (b)\\
    \caption{Average rewards of episodes for (a) music recommender without explicit states; (b) music recommender with states and state transitions. The solid lines are averaged from 5 runs, and the shaded areas are standard deviations.}\label{KKBox}
\end{figure}

\subsubsection{Simulation with states and transitions} 

The experimental setup for this experiment with states is as follows. We enlarge the size of the first hidden layer of the MLPs from 60 to 90 because there are more inputs (including the contexts and the states). UCB and TS here take the augmented contexts as input. They keep the same general linear modeled priors as the previous part. 

The result of the experiment is shown in figure \ref{KKBox}(b). PGCR outperforms other algorithms with larger map comparing with the previous experiment. An interesting fact is that both UCB and TS can only get almost the same return as in the previous experiment, which indicates that they can hardly make any use of the information from the states. PGCR learns faster and gets state-of-the-art performance in this task.

Since there are states and transitions, we also test the vanilla Policy Gradient method \cite{sutton2000policy} as another baseline, with the same MLP neural networks as actor and critic as PGCR. From the result, it shows that vanilla PG can outperform classic bandits methods, which is not surprising because it can make use of the state dynamics and maximize the long-term return.
However, it is still much worse than our purposed PGCR, for the reason that PGCR has a smaller search space and smaller variance when estimating the policy gradients. So the experimental results verified that PGCR is more sample efficient than the vanilla policy gradient method.

These simulation experiment results indicate that PGCR provides an alternative to conventional methods by using stochastic policies which can address the trade-off between exploration and exploitation well. The fast learning in the beginning phase and the stable performance over the entire training period of PGCR show that it is reliable to apply for real-world recommender systems.

\section{Conclusion and Discussion}
This paper has studied how to use the actor-critic algorithm with neural networks for general contextual recommendations, without unrealistic assumptions or prior knowledge to the problem. We first show that the class of permutation invariant policies is sufficient for our problem, and then derive the expected return of a policy depends on its marginal expected probability of choosing each item. We next propose a restricted class of policies in which the objective has a simple closed form and is differentiable to parameters. We prove that when using policies in this class, the gradient can be computed in closed-form. Furthermore, we propose Time-Dependent Greed and Actor-Dropout to significantly improve the performance and to guarantee the convergence property. Eventually, it comes to our proposed PGCR algorithm. The algorithm can be applied to standard contextual bandits as well as the generalized sequential decision-making problems with state and state transitions.

By testing on a toy dataset and a recommendation dataset, we showed that PGCR indeed achieves state-of-the-art performance for both classic one-step recommendations and MDP-CR with state transitions in a real-world scenario. 

It is a promising direction for the future work to extend our results to a variant of realistic recommendation settings, i.e, online advertising systems that choose multiple items at each step, or learning-to-rank that carries out diverse recommendation. 
\section*{Acknowledgement}
The work was supported by the National Key Research and Development Program of China under Grant No. 2018YFB1004300, the National Natural Science Foundation of China under Grant No. U1836206, U1811461, 61773361, the Project of Youth Innovation Promotion Association CAS under Grant No. 2017146.

Pingzhong Tang and Qingpeng Cai were supported in part by the National Natural Science Foundation of China Grant 61561146398, a China Youth 1000-talent program and an Alibaba Innovative Research program.
\bibliographystyle{ACM-Reference-Format}
\bibliography{bib}


\begin{thebibliography}{31}


\ifx \showCODEN    \undefined \def \showCODEN     #1{\unskip}     \fi
\ifx \showDOI      \undefined \def \showDOI       #1{#1}\fi
\ifx \showISBNx    \undefined \def \showISBNx     #1{\unskip}     \fi
\ifx \showISBNxiii \undefined \def \showISBNxiii  #1{\unskip}     \fi
\ifx \showISSN     \undefined \def \showISSN      #1{\unskip}     \fi
\ifx \showLCCN     \undefined \def \showLCCN      #1{\unskip}     \fi
\ifx \shownote     \undefined \def \shownote      #1{#1}          \fi
\ifx \showarticletitle \undefined \def \showarticletitle #1{#1}   \fi
\ifx \showURL      \undefined \def \showURL       {\relax}        \fi
\providecommand\bibfield[2]{#2}
\providecommand\bibinfo[2]{#2}
\providecommand\natexlab[1]{#1}
\providecommand\showeprint[2][]{arXiv:#2}

\bibitem[\protect\citeauthoryear{Abbasi-Yadkori, P{\'a}l, and
  Szepesv{\'a}ri}{Abbasi-Yadkori et~al\mbox{.}}{2011}]%
        {abbasi2011improved}
\bibfield{author}{\bibinfo{person}{Yasin Abbasi-Yadkori},
  \bibinfo{person}{D{\'a}vid P{\'a}l}, {and} \bibinfo{person}{Csaba
  Szepesv{\'a}ri}.} \bibinfo{year}{2011}\natexlab{}.
\newblock \showarticletitle{Improved algorithms for linear stochastic bandits}.
  In \bibinfo{booktitle}{\emph{Advances in Neural Information Processing
  Systems}}. \bibinfo{pages}{2312--2320}.
\newblock


\bibitem[\protect\citeauthoryear{Abe, Biermann, and Long}{Abe
  et~al\mbox{.}}{2003}]%
        {abe2003reinforcement}
\bibfield{author}{\bibinfo{person}{Naoki Abe}, \bibinfo{person}{Alan~W
  Biermann}, {and} \bibinfo{person}{Philip~M Long}.}
  \bibinfo{year}{2003}\natexlab{}.
\newblock \showarticletitle{Reinforcement learning with immediate rewards and
  linear hypotheses}.
\newblock \bibinfo{journal}{\emph{Algorithmica}} \bibinfo{volume}{37},
  \bibinfo{number}{4} (\bibinfo{year}{2003}), \bibinfo{pages}{263--293}.
\newblock


\bibitem[\protect\citeauthoryear{Adam, Busoniu, and Babuska}{Adam
  et~al\mbox{.}}{2012}]%
        {adam2012experience}
\bibfield{author}{\bibinfo{person}{Sander Adam}, \bibinfo{person}{Lucian
  Busoniu}, {and} \bibinfo{person}{Robert Babuska}.}
  \bibinfo{year}{2012}\natexlab{}.
\newblock \showarticletitle{Experience replay for real-time reinforcement
  learning control}.
\newblock \bibinfo{journal}{\emph{IEEE Transactions on Systems, Man, and
  Cybernetics, Part C (Applications and Reviews)}} \bibinfo{volume}{42},
  \bibinfo{number}{2} (\bibinfo{year}{2012}), \bibinfo{pages}{201--212}.
\newblock


\bibitem[\protect\citeauthoryear{Agrawal and Goyal}{Agrawal and Goyal}{2013}]%
        {agrawal2013thompson}
\bibfield{author}{\bibinfo{person}{Shipra Agrawal} {and} \bibinfo{person}{Navin
  Goyal}.} \bibinfo{year}{2013}\natexlab{}.
\newblock \showarticletitle{Thompson sampling for contextual bandits with
  linear payoffs}. In \bibinfo{booktitle}{\emph{International Conference on
  Machine Learning}}. \bibinfo{pages}{127--135}.
\newblock


\bibitem[\protect\citeauthoryear{Auer, Cesa-Bianchi, and Fischer}{Auer
  et~al\mbox{.}}{2002}]%
        {auer2002finite}
\bibfield{author}{\bibinfo{person}{Peter Auer}, \bibinfo{person}{Nicolo
  Cesa-Bianchi}, {and} \bibinfo{person}{Paul Fischer}.}
  \bibinfo{year}{2002}\natexlab{}.
\newblock \showarticletitle{Finite-time analysis of the multiarmed bandit
  problem}.
\newblock \bibinfo{journal}{\emph{Machine learning}} \bibinfo{volume}{47},
  \bibinfo{number}{2-3} (\bibinfo{year}{2002}), \bibinfo{pages}{235--256}.
\newblock


\bibitem[\protect\citeauthoryear{Bouneffouf, Bouzeghoub, and
  Gan{\c{c}}arski}{Bouneffouf et~al\mbox{.}}{2012}]%
        {bouneffouf2012contextual}
\bibfield{author}{\bibinfo{person}{Djallel Bouneffouf}, \bibinfo{person}{Amel
  Bouzeghoub}, {and} \bibinfo{person}{Alda~Lopes Gan{\c{c}}arski}.}
  \bibinfo{year}{2012}\natexlab{}.
\newblock \showarticletitle{A contextual-bandit algorithm for mobile
  context-aware recommender system}. In \bibinfo{booktitle}{\emph{International
  Conference on Neural Information Processing}}. Springer,
  \bibinfo{pages}{324--331}.
\newblock


\bibitem[\protect\citeauthoryear{Cai, Filos-Ratsikas, Tang, and Zhang}{Cai
  et~al\mbox{.}}{2018}]%
        {cai2018reinforcement}
\bibfield{author}{\bibinfo{person}{Qingpeng Cai}, \bibinfo{person}{Aris
  Filos-Ratsikas}, \bibinfo{person}{Pingzhong Tang}, {and}
  \bibinfo{person}{Yiwei Zhang}.} \bibinfo{year}{2018}\natexlab{}.
\newblock \showarticletitle{Reinforcement Mechanism Design for e-commerce}. In
  \bibinfo{booktitle}{\emph{Proceedings of the 2018 World Wide Web Conference
  on World Wide Web}}. International World Wide Web Conferences Steering
  Committee, \bibinfo{pages}{1339--1348}.
\newblock


\bibitem[\protect\citeauthoryear{Chapelle and Li}{Chapelle and Li}{2011}]%
        {chapelle2011empirical}
\bibfield{author}{\bibinfo{person}{Olivier Chapelle} {and}
  \bibinfo{person}{Lihong Li}.} \bibinfo{year}{2011}\natexlab{}.
\newblock \showarticletitle{An empirical evaluation of thompson sampling}. In
  \bibinfo{booktitle}{\emph{Advances in neural information processing
  systems}}. \bibinfo{pages}{2249--2257}.
\newblock


\bibitem[\protect\citeauthoryear{Chu, Li, Reyzin, and Schapire}{Chu
  et~al\mbox{.}}{2011}]%
        {chu2011contextual}
\bibfield{author}{\bibinfo{person}{Wei Chu}, \bibinfo{person}{Lihong Li},
  \bibinfo{person}{Lev Reyzin}, {and} \bibinfo{person}{Robert Schapire}.}
  \bibinfo{year}{2011}\natexlab{}.
\newblock \showarticletitle{Contextual bandits with linear payoff functions}.
  In \bibinfo{booktitle}{\emph{Proceedings of the Fourteenth International
  Conference on Artificial Intelligence and Statistics}}.
  \bibinfo{pages}{208--214}.
\newblock


\bibitem[\protect\citeauthoryear{Filippi, Cappe, Garivier, and
  Szepesv{\'a}ri}{Filippi et~al\mbox{.}}{2010}]%
        {filippi2010parametric}
\bibfield{author}{\bibinfo{person}{Sarah Filippi}, \bibinfo{person}{Olivier
  Cappe}, \bibinfo{person}{Aur{\'e}lien Garivier}, {and} \bibinfo{person}{Csaba
  Szepesv{\'a}ri}.} \bibinfo{year}{2010}\natexlab{}.
\newblock \showarticletitle{Parametric bandits: The generalized linear case}.
  In \bibinfo{booktitle}{\emph{Advances in Neural Information Processing
  Systems}}. \bibinfo{pages}{586--594}.
\newblock


\bibitem[\protect\citeauthoryear{Gal and Ghahramani}{Gal and
  Ghahramani}{2016}]%
        {gal2016dropout}
\bibfield{author}{\bibinfo{person}{Yarin Gal} {and} \bibinfo{person}{Zoubin
  Ghahramani}.} \bibinfo{year}{2016}\natexlab{}.
\newblock \showarticletitle{Dropout as a Bayesian approximation: Representing
  model uncertainty in deep learning}. In
  \bibinfo{booktitle}{\emph{international conference on machine learning}}.
  \bibinfo{pages}{1050--1059}.
\newblock


\bibitem[\protect\citeauthoryear{Heess, Hunt, Lillicrap, and Silver}{Heess
  et~al\mbox{.}}{2015}]%
        {heess2015memory}
\bibfield{author}{\bibinfo{person}{Nicolas Heess}, \bibinfo{person}{Jonathan~J
  Hunt}, \bibinfo{person}{Timothy~P Lillicrap}, {and} \bibinfo{person}{David
  Silver}.} \bibinfo{year}{2015}\natexlab{}.
\newblock \showarticletitle{Memory-based control with recurrent neural
  networks}.
\newblock \bibinfo{journal}{\emph{arXiv preprint arXiv:1512.04455}}
  (\bibinfo{year}{2015}).
\newblock


\bibitem[\protect\citeauthoryear{Jaksch, Ortner, and Auer}{Jaksch
  et~al\mbox{.}}{2010}]%
        {jaksch2010near}
\bibfield{author}{\bibinfo{person}{Thomas Jaksch}, \bibinfo{person}{Ronald
  Ortner}, {and} \bibinfo{person}{Peter Auer}.}
  \bibinfo{year}{2010}\natexlab{}.
\newblock \showarticletitle{Near-optimal regret bounds for reinforcement
  learning}.
\newblock \bibinfo{journal}{\emph{Journal of Machine Learning Research}}
  \bibinfo{volume}{11}, \bibinfo{number}{Apr} (\bibinfo{year}{2010}),
  \bibinfo{pages}{1563--1600}.
\newblock


\bibitem[\protect\citeauthoryear{Kingma and Ba}{Kingma and Ba}{2014}]%
        {kingma2014adam}
\bibfield{author}{\bibinfo{person}{Diederik~P Kingma} {and}
  \bibinfo{person}{Jimmy Ba}.} \bibinfo{year}{2014}\natexlab{}.
\newblock \showarticletitle{Adam: A method for stochastic optimization}.
\newblock \bibinfo{journal}{\emph{arXiv preprint arXiv:1412.6980}}
  (\bibinfo{year}{2014}).
\newblock


\bibitem[\protect\citeauthoryear{Krause and Ong}{Krause and Ong}{2011}]%
        {krause2011contextual}
\bibfield{author}{\bibinfo{person}{Andreas Krause} {and}
  \bibinfo{person}{Cheng~S Ong}.} \bibinfo{year}{2011}\natexlab{}.
\newblock \showarticletitle{Contextual gaussian process bandit optimization}.
  In \bibinfo{booktitle}{\emph{Advances in Neural Information Processing
  Systems}}. \bibinfo{pages}{2447--2455}.
\newblock


\bibitem[\protect\citeauthoryear{Langford and Zhang}{Langford and
  Zhang}{2008}]%
        {langford2008epoch}
\bibfield{author}{\bibinfo{person}{John Langford} {and} \bibinfo{person}{Tong
  Zhang}.} \bibinfo{year}{2008}\natexlab{}.
\newblock \showarticletitle{The epoch-greedy algorithm for multi-armed bandits
  with side information}. In \bibinfo{booktitle}{\emph{Advances in neural
  information processing systems}}. \bibinfo{pages}{817--824}.
\newblock


\bibitem[\protect\citeauthoryear{Li, Chu, Langford, and Schapire}{Li
  et~al\mbox{.}}{2010}]%
        {li2010contextual}
\bibfield{author}{\bibinfo{person}{Lihong Li}, \bibinfo{person}{Wei Chu},
  \bibinfo{person}{John Langford}, {and} \bibinfo{person}{Robert~E Schapire}.}
  \bibinfo{year}{2010}\natexlab{}.
\newblock \showarticletitle{A contextual-bandit approach to personalized news
  article recommendation}. In \bibinfo{booktitle}{\emph{Proceedings of the 19th
  international conference on World wide web}}. \bibinfo{pages}{661--670}.
\newblock


\bibitem[\protect\citeauthoryear{Lillicrap, Hunt, Pritzel, Heess, Erez, Tassa,
  Silver, and Wierstra}{Lillicrap et~al\mbox{.}}{2015}]%
        {lillicrap2015continuous}
\bibfield{author}{\bibinfo{person}{Timothy~P Lillicrap},
  \bibinfo{person}{Jonathan~J Hunt}, \bibinfo{person}{Alexander Pritzel},
  \bibinfo{person}{Nicolas Heess}, \bibinfo{person}{Tom Erez},
  \bibinfo{person}{Yuval Tassa}, \bibinfo{person}{David Silver}, {and}
  \bibinfo{person}{Daan Wierstra}.} \bibinfo{year}{2015}\natexlab{}.
\newblock \showarticletitle{Continuous control with deep reinforcement
  learning}.
\newblock \bibinfo{journal}{\emph{arXiv preprint arXiv:1509.02971}}
  (\bibinfo{year}{2015}).
\newblock


\bibitem[\protect\citeauthoryear{May, Korda, Lee, and Leslie}{May
  et~al\mbox{.}}{2012}]%
        {may2012optimistic}
\bibfield{author}{\bibinfo{person}{Benedict~C May}, \bibinfo{person}{Nathan
  Korda}, \bibinfo{person}{Anthony Lee}, {and} \bibinfo{person}{David~S
  Leslie}.} \bibinfo{year}{2012}\natexlab{}.
\newblock \showarticletitle{Optimistic Bayesian sampling in contextual-bandit
  problems}.
\newblock \bibinfo{journal}{\emph{Journal of Machine Learning Research}}
  \bibinfo{volume}{13}, \bibinfo{number}{Jun} (\bibinfo{year}{2012}),
  \bibinfo{pages}{2069--2106}.
\newblock


\bibitem[\protect\citeauthoryear{Shani, Heckerman, and Brafman}{Shani
  et~al\mbox{.}}{2005}]%
        {shani2005mdp}
\bibfield{author}{\bibinfo{person}{Guy Shani}, \bibinfo{person}{David
  Heckerman}, {and} \bibinfo{person}{Ronen~I Brafman}.}
  \bibinfo{year}{2005}\natexlab{}.
\newblock \showarticletitle{An MDP-based recommender system}.
\newblock \bibinfo{journal}{\emph{Journal of Machine Learning Research}}
  \bibinfo{volume}{6}, \bibinfo{number}{Sep} (\bibinfo{year}{2005}),
  \bibinfo{pages}{1265--1295}.
\newblock


\bibitem[\protect\citeauthoryear{Silver, Lever, Heess, Degris, Wierstra, and
  Riedmiller}{Silver et~al\mbox{.}}{2014}]%
        {silver2014deterministic}
\bibfield{author}{\bibinfo{person}{David Silver}, \bibinfo{person}{Guy Lever},
  \bibinfo{person}{Nicolas Heess}, \bibinfo{person}{Thomas Degris},
  \bibinfo{person}{Daan Wierstra}, {and} \bibinfo{person}{Martin Riedmiller}.}
  \bibinfo{year}{2014}\natexlab{}.
\newblock \showarticletitle{Deterministic policy gradient algorithms}. In
  \bibinfo{booktitle}{\emph{ICML}}.
\newblock


\bibitem[\protect\citeauthoryear{Slivkins, Radlinski, and Gollapudi}{Slivkins
  et~al\mbox{.}}{2013}]%
        {slivkins2013ranked}
\bibfield{author}{\bibinfo{person}{Aleksandrs Slivkins}, \bibinfo{person}{Filip
  Radlinski}, {and} \bibinfo{person}{Sreenivas Gollapudi}.}
  \bibinfo{year}{2013}\natexlab{}.
\newblock \showarticletitle{Ranked bandits in metric spaces: learning diverse
  rankings over large document collections}.
\newblock \bibinfo{journal}{\emph{Journal of Machine Learning Research}}
  \bibinfo{volume}{14}, \bibinfo{number}{Feb} (\bibinfo{year}{2013}),
  \bibinfo{pages}{399--436}.
\newblock


\bibitem[\protect\citeauthoryear{Srinivas, Krause, Kakade, and Seeger}{Srinivas
  et~al\mbox{.}}{2012}]%
        {srinivas2012information}
\bibfield{author}{\bibinfo{person}{Niranjan Srinivas}, \bibinfo{person}{Andreas
  Krause}, \bibinfo{person}{Sham~M Kakade}, {and} \bibinfo{person}{Matthias~W
  Seeger}.} \bibinfo{year}{2012}\natexlab{}.
\newblock \showarticletitle{Information-theoretic regret bounds for gaussian
  process optimization in the bandit setting}.
\newblock \bibinfo{journal}{\emph{IEEE Transactions on Information Theory}}
  \bibinfo{volume}{58}, \bibinfo{number}{5} (\bibinfo{year}{2012}),
  \bibinfo{pages}{3250--3265}.
\newblock


\bibitem[\protect\citeauthoryear{Sutton and Barto}{Sutton and Barto}{1998}]%
        {sutton1998reinforcement}
\bibfield{author}{\bibinfo{person}{Richard~S Sutton} {and}
  \bibinfo{person}{Andrew~G Barto}.} \bibinfo{year}{1998}\natexlab{}.
\newblock \bibinfo{booktitle}{\emph{Reinforcement learning: An introduction}}.
  Vol.~\bibinfo{volume}{1}.
\newblock \bibinfo{publisher}{MIT press Cambridge}.
\newblock


\bibitem[\protect\citeauthoryear{Sutton, McAllester, Singh, and Mansour}{Sutton
  et~al\mbox{.}}{2000}]%
        {sutton2000policy}
\bibfield{author}{\bibinfo{person}{Richard~S Sutton}, \bibinfo{person}{David~A
  McAllester}, \bibinfo{person}{Satinder~P Singh}, {and}
  \bibinfo{person}{Yishay Mansour}.} \bibinfo{year}{2000}\natexlab{}.
\newblock \showarticletitle{Policy gradient methods for reinforcement learning
  with function approximation}. In \bibinfo{booktitle}{\emph{Advances in neural
  information processing systems}}. \bibinfo{pages}{1057--1063}.
\newblock


\bibitem[\protect\citeauthoryear{Taghipour and Kardan}{Taghipour and
  Kardan}{2008}]%
        {taghipour2008hybrid}
\bibfield{author}{\bibinfo{person}{Nima Taghipour} {and} \bibinfo{person}{Ahmad
  Kardan}.} \bibinfo{year}{2008}\natexlab{}.
\newblock \showarticletitle{A hybrid web recommender system based on
  q-learning}. In \bibinfo{booktitle}{\emph{Proceedings of the 2008 ACM
  symposium on Applied computing}}. ACM, \bibinfo{pages}{1164--1168}.
\newblock


\bibitem[\protect\citeauthoryear{Tang, Jiang, Li, and Li}{Tang
  et~al\mbox{.}}{2014}]%
        {tang2014ensemble}
\bibfield{author}{\bibinfo{person}{Liang Tang}, \bibinfo{person}{Yexi Jiang},
  \bibinfo{person}{Lei Li}, {and} \bibinfo{person}{Tao Li}.}
  \bibinfo{year}{2014}\natexlab{}.
\newblock \showarticletitle{Ensemble contextual bandits for personalized
  recommendation}. In \bibinfo{booktitle}{\emph{Proceedings of the 8th ACM
  Conference on Recommender Systems}}. ACM, \bibinfo{pages}{73--80}.
\newblock


\bibitem[\protect\citeauthoryear{Tang, Jiang, Li, Zeng, and Li}{Tang
  et~al\mbox{.}}{2015}]%
        {tang2015personalized}
\bibfield{author}{\bibinfo{person}{Liang Tang}, \bibinfo{person}{Yexi Jiang},
  \bibinfo{person}{Lei Li}, \bibinfo{person}{Chunqiu Zeng}, {and}
  \bibinfo{person}{Tao Li}.} \bibinfo{year}{2015}\natexlab{}.
\newblock \showarticletitle{Personalized recommendation via parameter-free
  contextual bandits}. In \bibinfo{booktitle}{\emph{Proceedings of the 38th
  International ACM SIGIR Conference on Research and Development in Information
  Retrieval}}. ACM, \bibinfo{pages}{323--332}.
\newblock


\bibitem[\protect\citeauthoryear{Tang, Rosales, Singh, and Agarwal}{Tang
  et~al\mbox{.}}{2013}]%
        {tang2013automatic}
\bibfield{author}{\bibinfo{person}{Liang Tang}, \bibinfo{person}{Romer
  Rosales}, \bibinfo{person}{Ajit Singh}, {and} \bibinfo{person}{Deepak
  Agarwal}.} \bibinfo{year}{2013}\natexlab{}.
\newblock \showarticletitle{Automatic ad format selection via contextual
  bandits}. In \bibinfo{booktitle}{\emph{Proceedings of the 22nd ACM
  international conference on Conference on information \& knowledge
  management}}. ACM, \bibinfo{pages}{1587--1594}.
\newblock


\bibitem[\protect\citeauthoryear{Thompson}{Thompson}{1933}]%
        {thompson1933likelihood}
\bibfield{author}{\bibinfo{person}{William~R Thompson}.}
  \bibinfo{year}{1933}\natexlab{}.
\newblock \showarticletitle{On the likelihood that one unknown probability
  exceeds another in view of the evidence of two samples}.
\newblock \bibinfo{journal}{\emph{Biometrika}} \bibinfo{volume}{25},
  \bibinfo{number}{3/4} (\bibinfo{year}{1933}), \bibinfo{pages}{285--294}.
\newblock


\bibitem[\protect\citeauthoryear{Whittle}{Whittle}{1988}]%
        {whittle1988restless}
\bibfield{author}{\bibinfo{person}{Peter Whittle}.}
  \bibinfo{year}{1988}\natexlab{}.
\newblock \showarticletitle{Restless bandits: Activity allocation in a changing
  world}.
\newblock \bibinfo{journal}{\emph{Journal of applied probability}}
  \bibinfo{volume}{25}, \bibinfo{number}{A} (\bibinfo{year}{1988}),
  \bibinfo{pages}{287--298}.
\newblock


\end{thebibliography}

\end{document}